\documentclass{article}
\usepackage{lmodern}
\usepackage{amsmath,amsthm}
\usepackage{graphicx}

\newtheorem{theorem}{Theorem}
\usepackage{natbib}
\usepackage{array}
\usepackage{verbatim}
\usepackage{cprotect}
\usepackage{amsmath}
\usepackage{amsthm}
\usepackage{amssymb}
\usepackage{xcolor}

\makeatletter

\theoremstyle{plain}

\theoremstyle{definition}

\theoremstyle{remark}

\makeatother

\usepackage[english]{babel}
\providecommand{\examplename}{Example}
\providecommand{\remarkname}{Remark}
\providecommand{\theoremname}{Theorem}

\begin{document}
\title{Accuracy estimation of neural networks by extreme value theory}
\author{Gero Junike\thanks{Corresponding author. Department of Mathematics, Ludwig-Maximilians
Universität, Theresienstr. 39, 80333 München, Germany. E-mail: gero.junike@math.lmu.de}, Marco Oesting\thanks{Stuttgart Center for Simulation Science (SC SimTech) and Institute for Stochastics and Applications, University of Stuttgart,
70563 Stuttgart, Germany}}
\maketitle
\begin{abstract}
Neural networks are able to approximate any continuous function on
a compact set. However, it is not obvious how to quantify the error
of the neural network, i.e., the remaining bias between the function
and the neural network. Here, we propose the application of extreme
value theory to quantify large values of the error, which are typically
relevant in applications. The distribution of the error beyond some
threshold is approximately generalized Pareto distributed. We provide
a new estimator of the shape parameter of the Pareto distribution
suitable to describe the error of neural networks. 
Numerical experiments are provided.
\end{abstract}

\section{Introduction}

By the classical universal approximation theorem, any continuous function
$f$ mapping a compact subset $C$ of the $d$-dimensional space to
the reals can be arbitrarily closely approximated by a neural network
$\varphi$ with sufficient number of neurons and suitable activation
functions. 

However, except in rare cases, the bias between a neural network
with finitely many neurons and the function $f$ is not known. Let
\begin{equation}
\mathcal{E}(\omega)=|f(\omega)-\varphi(\omega)|,\quad\omega\in C \label{eq:gobal_error}
\end{equation}
describe the absolute error of the neural network approximating the
function~$f$. Researchers have reported the absolute error, the mean
square error and sometimes the maximal error between the neural network
$\varphi$ and $f$ on a test set, i.e., a finite subset of $C$.
Unfortunately, $\mathcal{E}$ may take much larger values compared
to the maximal value on a finite test set. In sum, these quantities
provide little insight \emph{into what large values of $\mathcal{E}$ look
like.} 

Large values of $\mathcal{E}$ are of interest
in various applications, and we provide an application in finance below. In this
article, we propose to apply extreme value theory to quantify the
variable $\mathcal{E}$ statistically. Under some mild conditions,
the Pickands-Balkema-de Haan Theorem states that the conditional distribution
of $\mathcal{E}$ above a high threshold $u$ is approximated by a
generalized Pareto distribution with scale $\sigma(u)>0$ and shape $\gamma\in\mathbb{R}$.
That is, the statistical properties of $\mathcal{E}$ above a certain
threshold $u$ are well known provided $\sigma(u)$ and $\gamma$ can be reliably
estimated. From a theoretical point of view, we know that $\mathcal{E}$
is bounded, since $\mathcal{E}$ is a continuous function in $\omega$
and $C$ is compact. This implies that $\gamma \leq 0$. However, classical
approaches to estimating $\gamma$, such as maximum likelihood or moment-based approaches, often violate the constraint $\gamma<0$. In this article,
we propose a new way to estimate $\gamma$ such that $\gamma<0$ with
probability one. 

Extreme value theory allows us to estimate the distribution of $\mathcal{E}$ beyond the threshold $u$, i.e., $P(\mathcal{E}>x)$, for $x \geq u$, which can interpreted
as the probability of making an error greater than $x$. Further, this theory makes it possible to estimate the quantity $\mathbb{E}[\mathcal{E} - u \mid \mathcal{E}>u]$, i.e., the average size of the exceedance $\mathcal{E} - u$, given that the error $\mathcal{E}$ exceeds the threshold $u$. That is, extreme value theory allows us to quantify
statistically the error $\mathcal{E}$ beyond some threshold $u$.
We also briefly discuss how a bound of $P(\mathcal{E}>x)$ can be
estimated by Markov's inequality. 

We consider an application in finance wherein $\omega$
denotes some details of a financial contract such as the end date
of the contract etc., see Section \ref{sec:Numerical-experiments}
for details. The price of the contract is $f(\omega)$. However, $f(\omega)$
can often only be evaluated by slow Monte Carlo methods, see \cite{glasserman2004monte}.
Therefore, many researchers have proposed learning $f$ by a neural network
$\varphi$ offline and then using $\varphi$ as an approximation of
$f$ during high-frequency trading times, see e.g., \cite{de2018machine,liu2019pricing,liu2019neural}. Ruf and Wang, see \citep[p.~1]{ruf2019neural}, for example observe that ``more than one hundred papers
in the academic literature concern the use of artificial neural networks
(ANNs) for option pricing and hedging.'' The application of neural
networks is often motivated purely by gain in computational time.
The error $\mathcal{E}$ then corresponds to the amount of money by which we misprice
the contract, i.e., the amount of money we may lose by pricing the contract
using $\varphi$ instead of $f$. Financial institutions should be
very interested in quantifying this error. In a financial application,
$u$ may be somewhat less than a U.S.~cent, which is often the smallest
tradable quantity. $P(\mathcal{E}>x)$ can then be interpreted as
the probability of mispricing by more than $x$, and $\mathbb{E}[\mathcal{E}-u \mid \mathcal{E}>u]$
tells us how much we misprice the contract on average given that we exceed
the threshold $u$.

This article is structured as follows: In Section \ref{sec:Intro_setting},
we state the problem more formally. In Section \ref{sec:Extreme-value-theory},
we introduce extreme value theory and provide a new way of estimating
the shape parameter $\gamma$. We investigate an application in finance in Section \ref{sec:Numerical-experiments} after which Section \ref{sec:Conclusions}
concludes.

\section{\protect\label{sec:Intro_setting}Problem statement}

Let $\varphi:\mathbb{R}^{d}\to\mathbb{R}$ be a neural network approximating
some continuous function $f:\mathbb{R}^{d}\to\mathbb{R}$. Let $C_{\text{test}}\subset C_{\text{train}}\subset\mathbb{R}^{d}$ 
be two uncountable compact sets describing the training and test domains.
The training set consists of $M$ randomly chosen
samples in $C_{\text{train}}$ and the test set consists of $N$ randomly
chosen samples in $C_{\text{test}}$, drawn independently from the training sample. We provide a probabilistic view
of the test set: We interpret $C_{\text{test}}$ as a sample space,
use the Borel-$\sigma$-Algebra as event space and fix on it a probability
measure $P$. In the applications, $P$ is often the uniform
distribution on $C_{\text{test}}$, but we do not need this fact in
the remainder of the paper. We describe the absolute error by the
following random variable:
\begin{align*}
\mathcal{E}:C_{\text{test}} & \to\mathbb{R}\\
\omega & \mapsto|f(\omega)-\varphi(\omega)|.
\end{align*}
Let $\varepsilon_{1},...,\varepsilon_{N}\in[0,\infty)$ be $N$ independent
realizations of $\mathcal{E}$ observed from the test set. Usually,
the mean absolute error
\[
E_{1}=\frac{1}{N}\sum\nolimits_{i=1}^{N}\varepsilon_{i},
\]
the mean squared error
\[
E_{2}=\frac{1}{N}\sum\nolimits_{i=1}^{N}\varepsilon_{i}^{2}
\]
and the maximal error 
\[
E_{\infty}=\max_{i=1,\ldots,N}\varepsilon_{i}
\]
are reported. Here, it is important to note that, with a positive
probability, the error $\mathcal{E}$ at a randomly chosen additional
test point might be even larger than $E_{\infty}$. Perceiving $\varepsilon_{1},\ldots,\varepsilon_{N}$
as realizations of independent copies $\mathcal{E}_{1},\ldots,\mathcal{E}_{N}$
of $\mathcal{E}$ with unknown continuous distribution function $F$,
it is easy to see that
\begin{equation}
P\left(\mathcal{E}>\max_{i=1,\ldots,N}\mathcal{E}_{i}\right)=\frac{1}{N+1}.\label{eq:E>max}
\end{equation}
Thus, even in the case of a large number of samples from the test set, there is still a non-negligible
probability of $\mathcal{E}$ exceeding $E_{\infty}$. This observation
motivates some further analysis of the distribution of $\mathcal{E}$
close to and \emph{beyond} $E_{\infty}$. In particular, we are interested
in estimating the exceedance probability
\begin{equation}
P(\mathcal{E}>x),\quad x\geq u\label{eq:PropMax}
\end{equation}
and the mean excess 
\begin{equation}
\mathbb{E}[\mathcal{E}-u\mid \mathcal{E}>u]\label{eq:ConExp}
\end{equation}
for some ``large'' threshold $u$, i.e., $u$ close to the upper
endpoint 
\[
x^{*}=\sup\{x:P(\mathcal{E}\leq x)<1\}.
\]
The exceedance probability describes the probability that the error
of the neural networks is greater than some $x \geq u$. The mean excess is helpful in quantifying the expected excess given that the error is already greater
than $u$. 

In many applications, the neural network is defined on a compact set
and, being continuous, the error $\mathcal{E}$ is therefore bounded,
which means that $x^{*}<\infty$. Extreme value theory then allows
us to estimate $x^{*}$ and describe the behavior of the distribution
$F$ close to $x^{*}$. Applying extreme value theory, we find sharp
estimates for these two quantities (\ref{eq:PropMax}) and (\ref{eq:ConExp}). 

We also compare our results based on extreme value theory to the simpler
approach of estimating (\ref{eq:PropMax}) by Markov's Inequality, which, for $m=2$, is closely related to Chebyshev's inequality and which states: 
\begin{equation}
P(\mathcal{E}>x)\leq\frac{\mathbb{E}[|\mathcal{E}|^{m}]}{x^{m}}\approx\frac{\frac{1}{N}\sum_{i=1}^N \varepsilon_{i}^{m}}{x^{m}},\quad m\geq0\quad x>0.\label{eq:Markov}
\end{equation}

\section{\protect\label{sec:Extreme-value-theory}Extreme value theory}

Assume that $\mathcal{E}_{1},\ldots,\mathcal{E}_{N}$ are independent,
with distribution function $F$. %
By the famous Pickands-Balkema-de Haan Theorem, and under mild assumptions
on $F$, the distribution of the exceedances above a threshold $u$
can be approximated by a generalized Pareto distribution
\[
H_{\gamma,\sigma}(x)=1-\left(\max\left\{1+\gamma\frac{x}{\sigma}, 0\right\}\right)^{-1/\gamma},\quad x>0,
\]
with scale parameter $\sigma>0$ and shape parameter $\gamma\in\mathbb{R}$, i.e., 
\begin{equation} \label{eq:gpd-approx}
P(\mathcal{E}-u\leq x\mid\mathcal{E}>u)\approx H_{\gamma,\sigma(u)}(x),\quad0\leq x<x^{\ast}-u,
\end{equation}
for some nonnegative function $\sigma$ of $u$ as $u\uparrow x^{*}$,
see, for instance, \cite{EKM97} for details.

As argued in Section \ref{sec:Intro_setting}, the distribution
of the absolute error in neural networks typically possesses a finite
upper end point. This implies that the shape parameter $\gamma$ is
negative (or zero) and the scaling function $\sigma$ in the Pickands--Balkema--de
Haan Theorem is of the form 
\[
\sigma(u)=-\gamma(x^{*}-u),\quad u<x^{*},
\]
yielding the approximation 
\begin{equation}
P(\mathcal{E}>x\mid\mathcal{E}>u)\approx\left(1-\frac{x-u}{x^{*}-u}\right)^{-1/\gamma},\quad u \leq x < x^{*},\label{eq:conProp}
\end{equation}
which can then be used to assess (\ref{eq:PropMax}) and (\ref{eq:ConExp})
provided that we have reliable estimates for $x^{*}$ and $\gamma$.

Here we follow \cite{FANR-17}, who constructed an estimator for
$x^{\ast}$ based on order statistics. More precisely, let $\varepsilon_{(1)}<\varepsilon_{(2)}<\ldots<\varepsilon_{(N)}$
be the sorted realizations of the absolute errors $\mathcal{E}_{1},\ldots,\mathcal{E}_{N}$.
Then, 
\[
\widehat{x^{*}}_{k,N}:=\varepsilon_{(N)}+\varepsilon_{(N-k)}-\frac{1}{\log(2)}\sum_{i=0}^{k-1}\log\big(1+\frac{1}{k+i}\big)\varepsilon_{(N-k-i)}
\]
is an estimator for $x^{*}$, which is consistent if $k(N)\to\infty$
and $k(N)/N\to0$, as $N\to\infty$, see \cite{FANR-17}.


As classical estimation techniques for $\gamma$ based on generalized
extreme value distributions or generalized Pareto distributions often
fail to meet the constraint $\gamma<0$, we derive a new estimator
for $\gamma$ based on a maximum-likelihood estimator in the following
theorem: 
\begin{theorem}\label{Thm1}
Let $\gamma <0$ be the extreme value index of $\mathcal{E}$ and $x^*$ be the corresponding upper end point. Furthermore, let
\begin{equation}
\widetilde{\gamma}_{k,N}:=\frac{1}{k}\sum_{j=0}^{k-1}\log\left(1-\frac{\varepsilon_{(N-j)}-\varepsilon_{(N-k)}}{x^{*}-\varepsilon_{(N-k)}}\right).\label{eq:Thm}
\end{equation}
Then $\widetilde{\gamma}_{k,N}$ is negative with probability one
and converges to $\gamma$ in probability for $N\to\infty$.
\end{theorem}

\begin{proof}
The proof can be found in the appendix.
\end{proof}
In practical applications, we propose to replace $x^{*}$ in Eq.~(\ref{eq:Thm})
by $\widehat{x^{*}}_{k,N}$. We denote the resulting estimator for $\gamma$ by $\widehat{\gamma}_{k,N}$.

In the remainder of this section, let $u:=\varepsilon_{(N-k)}$ for a suitable
$k$. In practice, one often chooses the value $k$ such that the
empirical exceedance probability $k/N$ provides a reliable estimate
for $P(\mathcal{E}>u)$, e.g., $k/N\approx0.01$ is often a reasonable choice.
Then, plugging the above estimators for $\gamma$ and $x^\ast$ into (\ref{eq:conProp}), we
obtain the approximations

\begin{align}
P(\mathcal{E}>x)= & P(\mathcal{E}>u)\cdot P(\mathcal{E}>x\mid\mathcal{E}>u)\nonumber \\
\approx & \frac{k}{N}\left(1-\frac{x-u}{\widehat{x^{*}}_{k,N}-u}\right)^{-1/\widehat{\gamma}_{k,N}},\quad u \leq x < \widehat{x^{*}}_{k,N},\label{eq:Pgx}
\end{align}
and
\begin{align}
\mathbb{E}[\mathcal{E}- u \mid \mathcal{E}>u] & =\int_{0}^{\infty}P(\mathcal{E}>y\mid\mathcal{E}>u)\,\mathrm{d}y-u\nonumber \\
 & \approx \int_{u}^{\widehat{x^{*}}_{k,N}}\left(1-\frac{y-u}{\widehat{x^{*}}_{k,N}-u}\right)^{-1/\widehat{\gamma}_{k,N}}\,\mathrm{d}y\nonumber \\
 & = \frac{\widehat{x^{*}}_{k,N}-u}{1-\frac{1}{\widehat{\gamma}_{k,N}}},\label{eq:EMinusUgU}
\end{align}
for (\ref{eq:PropMax}) and (\ref{eq:ConExp}), respectively.

\section{\protect\label{sec:Numerical-experiments}Numerical experiments}

In Section \ref{subsec:Markov-Inequality}, we apply Markov's inequality
to estimate (\ref{eq:PropMax}). From a theoretical point of view, Markov's inequality is much simpler than the application of extreme value
theory. However, Markov's inequality only provides an upper bound
for (\ref{eq:PropMax}) and is usually not very sharp. We will see
in Section \ref{subsec:Extreme-value-theory} that extreme value theory
is much better suited to estimate exceedance probabilities. Extreme
value theory also enables us to estimating (\ref{eq:ConExp}), which
cannot be archived by Markov's inequality.

\subsection{\protect\label{subsec:Markov-Inequality}Markov's Inequality}

\cite[Sec.~4.4.1]{liu2019pricing} use neural networks to price rapidly financial
contracts, which are usually priced with (computationally slow) Monte Carlo or Fourier
pricing techniques. They
employ an advanced model widely used in industry, the Heston model (see \cite{heston1993closed}), and obtain the following errors on a test set: $E_{1}=9.51\times10^{-5}$
and $E_{2}=1.65\times10^{-8}$. A maximal error is not reported. The
variable $\mathcal{E}$ has a financial interpretation: it describes
the absolute difference between the true price of the contract and
the approximation by the neural network, i.e., the financial mispricing
if the contract is priced by the neural network. Prices are typically rounded to the nearest whole unit in U.S.~cents. Therefore, let us assume that
$\mathcal{E}$ should be less than one-third of one U.S.~cent. Applying Markov's
inequality with $m=2$, we conclude that the probability that $\mathcal{E}$ is
greater than $0.0033$ is less than $0.15\%$. Put differently: with
(only) a probability of $99.85\%$, we can be sure that we make
a pricing error of less than one-third of one U.S.~cent. This probability might be
too far away from one in practical applications, since typically millions of such contracts are traded. Mispricing about $0.15\%$ of the contracts might result in a great loss. One way to solve this challenge would be to improve the neural network by using a larger training set. One would then decrease the mean square
error $E_{2}$ and yield better bounds applying Markov's inequality with $m=2$.
Alternatively, one could apply extreme value theory to estimate (\ref{eq:PropMax})
more precisely. 


\subsection{\protect\label{subsec:Extreme-value-theory}Extreme value theory}

As in \cite[Sec. 3.2.2]{de2018machine}, we use machine learning techniques
to price financial contracts called \emph{American put options}. We
first describe how the training and the test sets are constructed
and then apply extreme value theory to estimate (\ref{eq:PropMax})
and (\ref{eq:ConExp}). 

The price in U.S. Dollars of the American put option can be described
by a function $f(\omega)$, where 
\[
\omega=(K,T,r,q,\sigma)\in\mathbb{R}^{5}
\]
is specified in the contract and has the following interpretation
(see \cite{de2018machine} for details): The contract gives the holder
the right to sell a stock (which is fixed in the contract) anytime
before the maturity $T$ (in months) for a fixed price $K$ (in $\%$
of the stock price at the beginning of the contract). It is assumed
that the stock pays a dividend yield $q$ and can be described by
a binominal tree with volatility parameter $\sigma>0$. In order to be able to discount future cash-flows of the contract, it is assumed that there is a
risk-free bank account paying interest rates $r$. As in \cite[Sec. 3.2.2]{de2018machine},
we define the following training and test domains. Let
\[
C_{\text{train}}=[40\%,160\%]\times[11m,12m]\times[1.5\%,2.5\%]\times[0\%,5\%]\times[0.05,0.55]
\]
and
\[
C_{\text{test}}=[50\%,150\%]\times[11m,12m]\times[1.5\%,2.5\%]\times[0\%,5\%]\times[0.1,0.5].
\]
One can observe that $C_{\text{test}}$ is slightly smaller than
$C_{\text{train}}$. The reason is that many machine learning techniques
do not perform very well close to the boundary of the training domain. We uniformly
sample $10^{5}$ times from $C_{\text{train}}$ and price the contract
for each sample using a (slow) binominal tree. Using Gaussian regression for pricing an American put option is up to 137 times faster
than pricing using a binominal tree, see \cite{de2018machine}, which could confer a significant
advantage in high-frequency trading. 

We train a neural network $\varphi$ with three hidden layers consisting
of $300$ neurons each using the Adam optimizer. We use $20$ epochs,
the batch size is $100$ and we use $20\%$ of the data for validation.

Similarly, we generate $100$ independent test sets, each of size
$N=10^{5}$, by sampling uniformly from $C_{\text{test}}$. On each
test set, we apply extreme value theory to estimate the quantities
(\ref{eq:PropMax}) and (\ref{eq:ConExp}), as explained in Section \ref{sec:Extreme-value-theory}. 

Since prices are typically rounded to the nearest whole unit in U.S.
cents, we set $k=270$, which corresponds to a threshold of about $u=0.33$ of one U.S.~cent, throughout all test sets.
We estimate the probability of exceedance
(on a single test set) and obtain $P(\mathcal{E}>u)=0.26\%\pm0.03$,
which is very close to the true probability given by $0.25\%$. If
we are unlucky and we make a pricing error greater than the threshold
$u$, we estimate the mean excess using extreme value theory by Equation~(\ref{eq:EMinusUgU}) by $\mathbb{E}[\mathcal{E}-u\mid\mathcal{E}>u]=0.03\pm0.003$
U.S.~cents, which is almost identical to the empirical mean excess
estimated from all test sets together. In conclusion, the probability
of mispricing the contract by more than $0.33$ U.S.~cents is small
$(0.26\%)$, and if we misprice the contract by more than $0.33$ U.S.
cents, on average, we misprice it by $0.36$ U.S.~cents.

In Figure \ref{fig:1}, we see for different levels $x$  the probability of exceedance, i.e., $P(\mathcal{E}>x)$, estimated by extreme value theory by Equation~(\ref{eq:Pgx}) on average over all test sets including confidence intervals. Confidence intervals are obtained by adding and subtracting twice the standard deviation of the probability of exceedance. We compare the estimated
probabilities to the ``true'' probabilities, which are empirically
estimated using all 100 test sets together. We observe that extreme
value theory estimates the true exceedance probabilities precisely
for a wide range of values for $x$, i.e., for levels of $x$ between the
0.25\% quantile and the 0.01\% quantile. Beyond a certain point — above the 0.01\% quantile — the exceedance probability cannot be estimated reliably. This is because the estimation of $x^\ast$ is subject to uncertainties.

Markov bounds are also included in Figure \ref{fig:1}. These bounds
overestimate the true exceedance probabilities by a large margin, and we observe that Markov's inequality becomes sharper when four instead of two moments are used.

\begin{figure}
\begin{centering}
\includegraphics[scale=0.6]{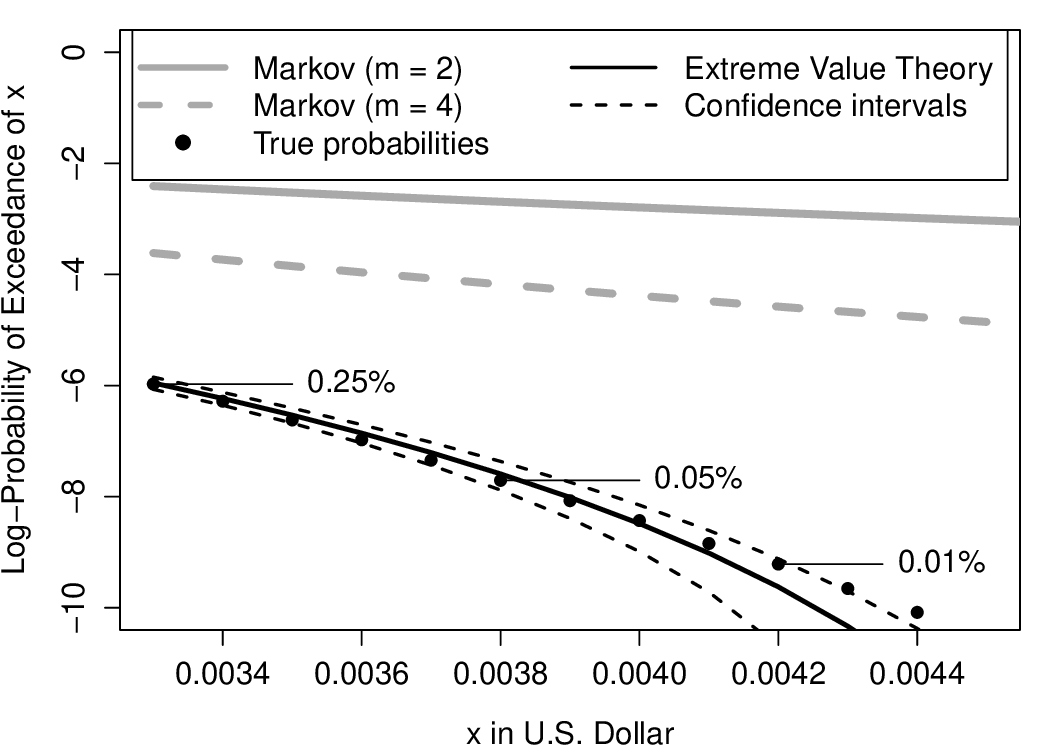}
\par
\end{centering}

\caption{\label{fig:1}Estimation of the probability of exceedance, $P(\mathcal{E}>x)$,
by extreme value theory and Markov's inequality. We use the threshold $u=0.33$ U.S.~cents by setting $k=270$.}
\end{figure}

\section{\protect\label{sec:Conclusions}Conclusions}

We analyze the error $\mathcal{E}$ beyond a certain threshold $u$
approximating a function $f$ by a neural network by extreme value
theory. The probability of exceedance and
the mean excess can be reliably estimated from a small test set for
a wide range of levels of $x$. In applications, large values of $\mathcal{E}$
are more critical. The probability of exceedance and the mean excess
help to quantify large values of $\mathcal{E}$ statistically. This analysis
has possible applications for risk management in financial institutions.

\bibliographystyle{plain}
\bibliography{biblio}

\appendix

\section{Proof of Theorem \ref{Thm1}} 
We choose the $k$th upper order statistic $\varepsilon_{(N-k)}$
as threshold $u$ in Equation \eqref{eq:gpd-approx}. Thus, we obtain the approximate likelihood 
\begin{align*}
L= & \prod_{j=0}^{k-1}\frac{\mathrm{d}}{\mathrm{d}x}\left[1-\left(1-\frac{x-\varepsilon_{(N-k)}}{x^{\ast}-\varepsilon_{(N-k)}}\right)^{-1/\gamma}\right]\Bigg|_{x=\varepsilon_{(N-j)}}\\
= & \prod_{j=0}^{k-1}-\frac{1}{\gamma\big(x^{\ast}-\varepsilon_{(N-k)}\big)}\left(1-\frac{\varepsilon_{(N-j)}-\varepsilon_{(N-k)}}{x^{\ast}-\varepsilon_{(N-k)}}\right)^{-1/\gamma-1}.
\end{align*}
Setting the derivative of the log-likelihood to zero, we obtain the
following maximum likelihood estimator for $\gamma$: 
\begin{equation*}
\widetilde{\gamma}_{k,N}=\frac{1}{k}\sum_{j=0}^{k-1}\log\left(1-\frac{\varepsilon_{(N-j)}-\varepsilon_{(N-k)}}{x^{\ast}-\varepsilon_{(N-k)}}\right)=-\frac{1}{k}\sum_{j=0}^{k-1}\log\left(\frac{(x^{*}-\varepsilon_{(N-j)})^{-1}}{(x^{*}-\varepsilon_{(N-k)})^{-1}}\right).
\end{equation*}
The estimator $\widetilde{\gamma}_{k,N}$ is less than zero with probability
one. It is well-known from univariate extreme value theory that $\mathcal{E}$
is in the max-domain of attraction of a Weibull distribution with
shape parameter $-1/\gamma$ if and only if $1/(x^{*}-\mathcal{E})$
is in the max-domain of attraction of a Fréchet distribution with
parameter $-1/\gamma$, see, for instance, Theorem 3.3.12 in \cite{EKM97},
i.e., $\widetilde{\gamma}_{k,N}$ is the negative Hill estimator for
the random variable $1/(x^{*}-\mathcal{E})$. Consequently, 
\[
\widetilde{\gamma}_{k,N}\to_{p}\gamma
\]
as $N\to\infty$, see, for instance, Theorem 4.2 in \cite{resnick07}. 
\end{document}